\documentclass[11pt]{article}
\usepackage{fullpage}
\usepackage[numbers]{natbib}
\usepackage{algorithm}
\usepackage[noend]{algorithmic}
\usepackage{amsmath,amsthm,amsfonts,amssymb}
\usepackage{amsmath}
\usepackage{hyperref}
\usepackage{color}
\usepackage{xcolor}
\usepackage{mathrsfs}
\usepackage{enumitem}
\usepackage{bm}
\usepackage{multirow}
\usepackage{booktabs}
\usepackage{makecell}
\usepackage{graphicx}
\usepackage{subfigure}
\usepackage{caption}
\usepackage{thmtools}
\usepackage{thm-restate}
\usepackage{setspace}
\usepackage{makecell}
\definecolor{light-gray}{gray}{0.85}
\usepackage{colortbl}
\usepackage{hhline}
\usepackage{accents}

\newcommand{\defeq}{\mathrel{\mathop:}=}

\newcommand{\argmax}{\mathop{\rm argmax}}

\newcommand{\trans}{^{\top}}

\newcommand{\poly}{\mathrm{poly}}

\newcommand{\E}{\mathbb{E}}

\renewcommand{\P}{\mathbb{P}}

\newcommand{\cO}{\mathcal{O}}

\newcommand{\N}{\mathbb{N}}
\newcommand{\R}{\mathbb{R}}

\newcommand{\cS}{\mathcal{S}}
\newcommand{\cA}{\mathcal{A}}

\newenvironment{proof-sketch}{\noindent{\bf Proof Sketch}
  \hspace*{1em}}{\qed\bigskip\\}
\newenvironment{proof-idea}{\noindent{\bf Proof Idea}
  \hspace*{1em}}{\qed\bigskip\\}
\newenvironment{proof-of-lemma}[1][{}]{\noindent{\bf Proof of Lemma {#1}}
  \hspace*{1em}}{\qed\bigskip\\}
\newenvironment{proof-of-proposition}[1][{}]{\noindent{\bf
    Proof of Proposition {#1}}
  \hspace*{1em}}{\qed\bigskip\\}
\newenvironment{proof-of-theorem}[1][{}]{\noindent{\bf Proof of Theorem {#1}}
  \hspace*{1em}}{\qed\bigskip\\}
\newenvironment{inner-proof}{\noindent{\bf Proof}\hspace{1em}}{
  $\bigtriangledown$\medskip\\}
\newenvironment{proof-attempt}{\noindent{\bf Proof Attempt}
  \hspace*{1em}}{\qed\bigskip\\}

  \newcommand{\Reg}{{\rm Regret}}
  
  \renewcommand{\a}{\mathbf{a}}

\newcommand{\bp}{\mathbf{p}}
\newcommand{\NL}{N^{\rm lazy}}
\newcommand{\up}[1]{\mkern 1.5mu\overline{\mkern-1.5mu#1\mkern-1.5mu}\mkern 1.5mu}
\renewcommand{\hat}{\widehat}

\newcommand{\BR}{{\rm BR}}
\newcommand{\OPMD}{\textsf{OP-EXP3}}
\newcommand{\AOPMD}{\textsf{adaptive}~\textsf{OP-EXP3}}

\usepackage{times}

\newtheorem{theorem}{Theorem}
\newtheorem{lemma}[theorem]{Lemma}

\newtheorem{proposition}[theorem]{Proposition}

\theoremstyle{definition}
\newtheorem{definition}[theorem]{Definition}

\renewcommand{\up}[1]{\overline{#1}}

\renewcommand{\Reg}{{\rm Regret}}

\newfloat{subroutine}{htbp}{loa}
\floatname{subroutine}{Subroutine}

\title{\bf Learning Markov Games with Adversarial Opponents: \\Efficient Algorithms and Fundamental Limits}

\author{%
   Qinghua Liu$^{\star}$
       \and
       Yuanhao Wang$^{\star}$ 
       \and
       Chi Jin$^{\dagger}$ 
}
\date{}

\begin{document}

\maketitle

\renewcommand{\O}{\mathbb{O}}
\newcommand{\T}{\mathbb{T}}
\newcommand{\Ot}{\widetilde{\O}}
\newcommand{\mM}{\mathbf{M}}
\newcommand{\mN}{\mathbf{N}}
\newcommand{\mX}{\mathbf{X}}
\newcommand{\mB}{\mathbf{B}}
\newcommand{\mE}{\mathbf{E}}
\newcommand{\rowspan}[1]{ {\rm rowspan}(#1)}
\newcommand{\colspan}[1]{ {\rm colspan}(#1)}
\newcommand{\Mhat}{\widehat{\mathbf{M}}}
\newcommand{\Nhat}{\widehat{\mathbf{N}}}
\newcommand{\Bhat}{\widehat{\mathbf{B}}}
\newcommand{\Ahat}{\widehat{\mathbf{A}}}

\newcommand{\alglinelabel}{%
  \addtocounter{ALC@line}{-1}
  \refstepcounter{ALC@line}
  \label
}
\makeatletter
\def\blfootnote{\gdef\@thefnmark{}\@footnotetext}
\makeatother

\begin{abstract}

    An ideal strategy in zero-sum games should not only grant the player an average reward no less than the value of Nash equilibrium, but also exploit the (adaptive) opponents when they are suboptimal. While most existing works in Markov games focus exclusively on the former objective, it remains open whether we can achieve both objectives simultaneously. To address this problem, this work studies no-regret learning in Markov games with adversarial opponents when \emph{competing against the best fixed policy in hindsight}. Along this direction, we present a new complete set of positive and negative results:
    
    When the policies of the opponents are revealed at the end of each episode, we propose new efficient algorithms achieving $\sqrt{K}$-regret bounds when either (1) the baseline policy class is small or (2) the opponent’s policy class is small. This is complemented with an exponential lower bound when neither conditions are true. When the policies of the opponents are not revealed, we prove a statistical hardness result even in the most favorable scenario when both above conditions are true. Our hardness result is much stronger than the existing hardness results which either only involve computational hardness, or require further restrictions on the algorithms.
    
    \end{abstract}


\section{Introduction}

\blfootnote{$^\star$ Equal contribution;  Princeton University; Email: 
   \texttt{\{qinghual,yuanhao\}@princeton.edu} }

\blfootnote{$^\dagger$ Princeton University; Email: 
   \texttt{\{chij\}@princeton.edu} }

Multi-agent reinforcement learning (MARL) studies how multiple players sequentially interact with each other and the environment to maximize the  cumulative rewards. 
Recent years have witnessed inspiring breakthroughs in the application of multi-agent reinforcement learning to various challenging AI tasks, including, but not limited to, GO \citep{silver2016mastering,silver2017mastering}, Poker \citep{brown2019superhuman}, real-time strategy games (e.g., StarCraft and Dota) \citep{vinyals2019grandmaster, openaidota},  autonomous driving \citep{shalev2016safe}, 
decentralized controls or multi-agent robotics systems \citep{brambilla2013swarm},
as well as complex social scenarios such as hide-and-seek~\citep{baker2020emergent}. 

Despite its great empirical success, MARL still suffers from limited theoretical understanding with many fundamental questions left open. Among them, one central and challenging question is how to exploit the (adaptive) suboptimal opponents while staying invulnerable to the optimal opponents. Achieving this objective requires a solution concept beyond Nash equilibria. As a motivating example, we consider the game of rock-paper-scissors with a suboptimal opponent who plays rock in the first $K/2$ games and then switches to paper in the next $K/2$ games. A strategic player in this case should be able to learn from the behavior of the opponent and exploit it to get a return of $\Omega(K)$. In contrast, playing a Nash equilibrium (which plays all actions uniformly) only yields an average return of zero.

In classical normal-form games (which can be viewed as special cases of MARL without transition and states), the question of exploiting adaptive opponents has been extensively studied under the framework of no-regret learning, where the agent is required to compete against the best fixed policy in hindsight even when facing adversarial opponents \citep[see e.g.,][]{cesa2006prediction}. On the other hand, addressing general MARL brings a number of new challenges such as unknown environment dynamics and sequential correlations between the player and the opponents. Consequently, all existing results \citep[e.g.,][]{brafman2002r,wei2017online,tian2021online,jin2021power} have only focused on competing against Nash equilibria when facing adversarial opponents. This motivates us to ask the following question for MARL:
\begin{center}
    \textbf{Can we compete against the best fixed policy in hindsight and achieve no-regret learning in MARL?}
\end{center}

In this paper, we consider two-player zero-sum Markov games \citep{shapley1953stochastic,littman1994markov} as a model for MARL, and address the above question by providing a complete set of positive and negative results as follows. We refer to \emph{general policies} as policies which can depend on the entire history, in contrast to \emph{Markov policies}, which can only depend on the state at the current step.

\paragraph{Statistical efficiency (standard setting).}

We first consider the most standard setting, in which only the actions of the opponents are observed, and prove an exponential  lower bound for the regret. 
Importantly, the lower bound holds even if the baseline policy class only contains Markov policies and the opponent only alternates between a small number of Markov policies. 
Besides, this hardness result is much stronger than the existing ones which either only involve computational hardness \citep{bai2020near}, or require further restrictions on the algorithms \citep{tian2021online}.
The proof of the lower bound builds upon the key observation that we can simulate any POMDP/latent MDP by a Markov game of similar size and an opponent playing general/Markov policies. This directly implies  no-regret learning in MGs is no easier than learning POMDPs/latent MDPs which is  statistically intractable in general \citep{jin2020sample,kwon2021rl}.

\paragraph{Statistical efficiency (revealed-policy setting).} 
Given that only observing actions of  opponents is insufficient for achieving sublinear regret, we then consider a setting more advantageous  to the learner, in which the opponent reveals the policy she played at the end of each episode.
    \begin{itemize}
 \item When baseline policies---the set of policies we are competing against in the definition of regret (see Definition \ref{def:reg})---are Markov policies, we propose \textsl{\textbf{O}ptimistic \textbf{P}olicy \textbf{EXP3}} (\OPMD, Algorithm  \ref{alg:md}) that has $\tilde{\cO}(\sqrt{H^4 S^2AK})$-regret even when the opponent can play arbitrary general (history-dependent) policies, where 
        $H$ is the length of each episode, $S$ is the number of states, $A$ is the number of actions, and $K$ is the number of episodes. 
        \item When baseline policies are general policies, We further propose  \AOPMD~(Algorithm \ref{alg:md-restart}) that achieves regret $\tilde\cO (\sqrt{H^4S^2AK}+\sqrt{|\Psi^\star|SAH^3K} +\sqrt{|\Psi^\star|^2H^2K})$ 
        when the opponent only chooses policies from an unknown policy class $\Psi^\star$.
        \item Finally, we complement our upper bounds with an exponential lower bound for competing against general policies, which holds even when the opponent only plays deterministic Markov policies. 
    \end{itemize}

    \paragraph{Computational efficiency.}
    Finally, we prove that achieving sublinear regret is computationally hard even in the very favorable setting where (a) the learner only competes against the best fixed Markov policy in hindsight, (b) the opponent only chooses policies randomly from a known small set of Markov policies and reveals the policy she played at the end of each episode, (c) the MG model is known. 
        We emphasize that this computational hardness holds under very weak conditions as stated above, and applies to all the settings studied in this paper.

To summarize, we provide a complete set of results  including both efficient algorithms and fundamental limits for no-regret learning in Markov games with adversarial opponents. 
We refer the reader to Table \ref{table:1} for a brief summary of our main results.

\begin{table}[t]
    \renewcommand{\arraystretch}{1.2} 
    \centering
    \resizebox{\columnwidth}{!}{
	\begin{tabular}{|c|c|c|c|}
		\hline
		\textbf{Baseline Policies}          & \textbf{Opponent's Policies} & \textbf{Standard Setting}                                       & \textbf{Revealed-policy Setting}                            \\ \hline\hline
		{Markov policies} &{General policies}  & \multirow{3}{*}{\shortstack{ $\Omega(\min\{K, 2^H\}/H)$}}            &   {$\tilde{\cO}(\sqrt{H^4 S^2AK})$} \\
         \cline{1-2}\cline{4-4}
		\multirow{2}{*}{General policies}   & Finite class $\Psi^\star$ &  &                                        $\tilde\cO (\sqrt{H^4S^2AK}+\sqrt{|\Psi^\star|SAH^3K} +\sqrt{|\Psi^\star|^2H^2K})$                                                     \\ \cline{2-2}\cline{4-4}
		& Markov policies            &&                                           $\Omega(\min\{K, 2^H\})$                                                  \\ \hline
	\end{tabular}}
\caption{A summary of the main results. 
Baseline policies refer to the policies the algorithm competes against in the definition of regret (see Definition \ref{def:reg}).
General policies include both Markov and history-dependent policies.\label{table:1}}
\end{table}

\section{Related Work}

\paragraph{Learning Nash equilibria in Markov games.} There has been a long line of literature focusing on learning the Nash equilibrium of Markov games when either the dynamics are known, or the amount of collected data goes to infinity~\citep{littman1994markov, hu2003nash,hansen2013strategy,lee2020linear}. Later works have considered self-play algorithms that incorporate exploration and can find Nash equilibrium in  Markov games with unknown dynamics~\citep{wei2017online,bai2020near,bai2020provable,xie2020learning,liu2021sharp}.

When the algorithm is only able to control one player and the other player is potentially adversarial, \citet{brafman2002r} proposed the R-max algorithm, and showed that it is able to obtain average value close to the Nash value. 
Later works \citep{wei2017online,tian2021online,jin2021power} obtain similar or improved results also for comparing to the Nash value.

\paragraph{Learning latent MDPs.} In latent MDPs, sometimes also referred as multi-model MDPs, a latent variable is drawn from a fixed distribution at the start of each episode, and the dynamics of the MDP would be a function of this latent variable.~\citet{steimle2021multi} has shown that finding the optimal Markov policy in the latent MDP problem is computational hard;~\citet{kwon2021rl} considered reinforcement learning in latent MDPs, providing both statistical lower bounds for the general case and sample complexity upper bounds under further assumptions. Latent MDPs, and in fact POMDPs \citep{smallwood1973optimal,azizzadenesheli2016reinforcement,jin2020sample} in general, can be simulated using Markov games with adversarial opponents as proved in this paper; thus learning latent MDPs can be viewed as a special case of the setting considered in this paper.

\paragraph{Adversarial MDPs.} Another line of work focuses on the single-agent adversarial MDP setting where the transition  or the reward function is adversarially chosen for each episode. 
When the adversary can arbitrarily alter the transition, 
\citet{abbasi2013online} prove that no-regret learning is computationally at least as hard as learning parity with noise.
Later work by \citet{bai2020near} adapt similar hard instance for  Markov games and prove that  achieving sublinear regret in MGs against adversarial opponents  is also computationally hard. 
On the other hand, if the transition is fixed and the adversary is only allowed to alter the reward function, sublinear regret can be achieved by various algorithms~\citep{jin2019learning,zimin2013online,rosenberg2019online,shani2020optimistic} in  competing against the best Markov policy in hindsight. 

\paragraph{Matrix games and extensive form games.} For matrix games, it is well known that playing EXP-style algorithms would allow one to compete with the best policy (action profile) in hindsight  \citep[see e.g.,][]{cesa2006prediction}. For extensive form games (EFGs), similar no-regret guarantees can be achieved via counterfactual regret minimization~\citep{zinkevich2007regret} or online convex optimization~\citep{gordon2007no,farina2020faster,farina2021model,kozuno2021model}. 
EFGs can be viewed a special subclass of MGs where the transition admits a strict tree structure. 
Therefore, results for EFGs do not directly apply to MGs.

\section{Preliminaries} \label{sec:prelim}
In this paper, we consider Markov Games \citep[MGs,][]{shapley1953stochastic, littman1994markov}, which generalize the standard Markov Decision Processes (MDPs) into the multi-player setting, where  each player seeks  maximizing her own utility.

Formally, we study the tabular episodic version of two-player zero-sum Markov games, which is specified by a tuple $( \cS, \cA, H, \P, r)$. Here $\cS$ denotes the state set with $|\cS| \le S$. $\cA=\cA_{\max}\times\cA_{\min}$ (with $|\cA| \le A$)  denotes the action-pair set that is equal to the Cartesian product of the action set of the max-player $\cA_{\max}$ and  the action set of the min-player $\cA_{\min}$.
 $H$ denotes the length of each episode.
$\P = \{\P_h\}_{h\in[H]}$ denotes a  collection of transition matrices, so that $\P_h ( \cdot | s, \a)$ gives the distribution of the next state if action-pair $\a\in\cA$ is taken at state $s$ at step $h$.
 $r = \{r_h\}_{h\in[H]}$ denotes a collection of expected reward functions, where $r_h \colon \cS \times \cA \to [0,1]$ is the expected reward function at step $h$. This reward represents both the gain of the max-player and the loss of the min-player, making the problem a zero-sum Markov game.
 For cleaner presentation, we assume the reward function is known in this work.\footnote{Our results immediately generalize to unknown reward functions effortlessly, since learning the transitions is more difficult than learning the rewards in tabular MGs.}

In each episode, the environment starts from a \emph{fixed initial state} $s_1$. At step $h \in [H]$, both
players observe state $s_h \in \cS$, and then pick their own actions $a_{h,\max} \in \cA_{\max}$ and $a_{h,\min} \in \cA_{\min}$ simultaneously. 
After that, both players observe the action of their opponent, receive reward
$r_h(s_h, \a_h)$, and then the environment transitions to the next state
$s_{h+1}\sim\P_h(\cdot | s_h, \a_h)$. The episode terminates immediately once
$s_{H+1}$ is reached. 

We use $\tau_h=(s_1,\a_1,\ldots,s_{h-1},\a_{h-1},s_h)\in(\cS\times\cA)^{h-1}\times\cS$ to denote a  trajectory from step $1$ to step $h$, which includes the state but excludes the action at step $h$. We use box brackets to denote the concatenation of trajectories, e.g., $[\tau_h,\a_h,s_{h+1}]\in(\cS\times\cA)^{h}\times\cS$ gives a  trajectory  from step $1$ to step $h+1$ by concatenating $\tau_h$ with an action-state pair $(\a_h,s_{h+1})$.

\paragraph{Policy.} 
We consider two classes of policies: Markov policies and general policies.
A \textbf{Markov policy} $\mu=\{ \mu_h: \cS \rightarrow \Delta_{\cA_{\max}} \}_{h\in [H]}$ of the max-player is a collection of $H$ functions, each mapping from a state to a distribution over actions. (Here $\Delta_{\cA_{\max}}$ is the probability simplex over action set $\cA_{\max}$.) Similarly, a Markov policy of the min-player is of form $\nu=\{ \nu_h: \cS \rightarrow \Delta_{\cA_{\min}} \}_{h\in [H]}$. 
Different from Markov policies, a \textbf{general policy} can choose actions depending on the entire history of interactions. Formally, a general policy  
$\mu=\{ \mu_h: (\cS\times\cA)^{h-1}\times\cS \rightarrow \Delta_{\cA_{\max}} \}_{h\in [H]}$ of the max-player is a collection of $H$ functions where each function maps a trajectory  to a distribution over actions.  The definition of general policies of the min-player follows similarly. We remark that Markov policies are special cases of general policies, which pick actions only conditioning on the current state.

\paragraph{Value function.} 
Given any pair of general policies $(\mu,\nu)$, 
 we use $V^{\mu\times \nu}_1(s_1)$ to denote its value function, 
 which is equal to the expected cumulative rewards
received by the max-player, if the game starts at state $s_1$ at the $1^{\rm th}$ step and the max-player and the min-player follow policy $\mu$ and $\nu$ respectively:
\begin{equation} \label{eq:V_value}
        \textstyle
 V^{\mu\times \nu}_1(s_1) \defeq \E_{\mu\times \nu}\left[\left.\sum_{h =
        1}^H r_{h}(s_{h}, \a_{h}) \right| s_1 \right],
\end{equation}
where the expectation is taken with respect to the randomness of $\a_1,s_2,\a_2,\ldots,s_H,\a_H$.

\paragraph{Best response and Nash equilibrium.}
Given any general policy of the max-player $\mu$, there exists a \emph{best response} of the min-player
$\nu^\dagger(\mu)$ so that  $V_1^{\mu \times \nu^\dagger(\mu)}(s_1) = \inf_{\nu} V_1^{\mu\times \nu}(s_1)$.
For brevity of notations, we denote $V_1^{\mu, \dagger} \defeq V_1^{\mu\times\nu^\dagger(\mu)}$. By symmetry, we
can also define $\mu^\dagger(\nu)$ and $V_1^{\dagger,\nu}$.   
Moreover, previous works \citep[e.g.,~][]{filar2012competitive} prove that there exist policies $\mu^\star$, $\nu^\star$
that are optimal against the best responses of the opponents, in the sense that
\begin{equation*}        \textstyle
  V^{\mu^\star, \dagger}_1(s_1) = \sup_{\mu}
      V^{\mu, \dagger}_1(s_1), 
      \quad  V^{\dagger, \nu^\star}_1(s_1) = \inf_{\nu}
      V^{\dagger, \nu}_1(s_1).
\end{equation*}
We refer to such strategies $(\mu^\star,\nu^\star)$ as the Nash equilibria of the Markov game. 
Importantly, any Nash equilibrium satisfies the following minimax theorem 
\footnote{We remark that  the minimax theorem for MGs is different from the one for matrix games, i.e. $\max_x\min_y x\trans A y = \min_y \max_x x\trans Ay$ for any matrix $A$, because $V^{\mu \times \nu}_1(s_1)$ is in general not bilinear in $\mu, \nu$.}:
\begin{equation*}        \textstyle
 \sup_{\mu} \inf_{\nu} V^{\mu\times\nu}_1(s_1) = V^{\mu^\star \times \nu^\star}_1(s_1) = \inf_{\nu} \sup_{\mu} V^{\mu\times \nu}_1(s_1).
\end{equation*}
The minimax theorem above directly implies the value function of Nash equilibria is unique, which we denote as $V_1^{\star}(s_1)$.
Furthermore, it is also known that there always exists a Markov Nash equilibrium in the sense that both $\mu^\star$ and $\nu^\star$ are Markov.
Intuitively, a Nash equilibrium gives a solution in which no player can benefit  from unilateral deviation.

\paragraph{Learning objective.}
In this work, we study no-regret learning of Markov games with adversarial opponents, and 
 measure the performance of an algorithm by its regret against the best fixed policy in hindsight from a prespecified set of policies. From now on, we refer to this policy set as the \emph{baseline policy class}, and denote it by $\Phi^\star$.
\begin{definition}[Regret]\label{def:reg}
  Let $(\mu^k$, $\nu^k)$ denote the policies deployed by the algorithm and the opponent 
  in the $k^{\text{th}}$ episode.
   After a total of $K$ episodes, the regret is defined as
  \begin{equation}\label{eq:reg}
    \Reg_{\Phi^\star}(K) = \max_{\mu\in\Phi^\star}\sum_{k=1}^K (V^{\mu\times\nu^k}_{1} - V^{\mu^k\times \nu^k}_{1}) (s_1).
  \end{equation}
\end{definition}
 When the baseline policy class $\Phi^\star$ includes all the general policies, we will omit subscript $\Phi^\star$ and simply write $\Reg(K)$.

 Compared to previous works \citep[e.g.,][]{brafman2002r,wei2017online,tian2021online,jin2021power} that only pursue achieving the Nash value, i.e., considering the following version of regret
 \begin{equation}\label{eq:pre-reg}        \textstyle
 \sum_{k=1}^K (V^{\star}_{1} - V^{\mu^k\times \nu^k}_{1}) (s_1),
 \end{equation}
our regret defined in \eqref{eq:reg} is a much stronger criterion because it forces the algorithm to exploit the opponents to achieve higher value than Nash equilibria whenever the opponent is exploitable. In stark contrast, the regret defined in \eqref{eq:pre-reg} only requires the algorithm itself to be  invulnerable. Moreover, if the baseline policy class includes all the Markov policies, then the regret defined in \eqref{eq:reg}  is an upper bound for the latter one  because there always exists a Markov Nash equilibrium as mentioned before.

Finally, observe that the regret defined in \eqref{eq:reg} does not depend on  the payoff function of the min-player, so it is still well-defined in the general-sum setting. 
Actually, all the results derived in  this work can be directly extended  to the general-sum setting, although the current paper assumes zero MGs for cleaner presentation and more direct comparison to previous works.


\section{Results for the Standard Setting}
\label{sec:action-hardness}

 In this section, we consider the standard setting where the opponent only reveals her actions to the learner during their interaction.  
We show that achieving low regret  in this setting is impossible in general even if 
(a) the baseline policy class consists of  Markov policies, \emph{and} (b) the opponent sticks to a fixed general policy or only alternates between $H$ different Markov policies. 
Our hardness results build on the generality of Markov games, i.e., the ability to simulate POMDPs and latent MDPs with specially designed opponents. 

\subsection{Against opponents playing a fixed  general policy}

To begin with, we show competing with the best Markov policy in hindsight is  statistically hard when the opponent keeps playing a fixed general  policy. 
\begin{theorem}
	\label{thm:action-general}
	There exists a Markov game with $S,A=\cO(1)$ and an opponent playing a fixed unknown general policy, such that the regret for competing with the best fixed Markov policy in hindsight is  $\Omega(\min\{K, 2^{H}\})$.
\end{theorem}

Theorem \ref{thm:action-general} claims that even in a Markov game of constant size, if the learner is only able to observe the opponent’s actions instead of the opponent’s policies, then there exists a regret lower bound exponential in the horizon length $H$ for competing with the best fixed Markov policy in hindsight when the opponent plays a fixed unknown general policy.

The proof relies on the fact that a POMDP can be simulated by a Markov game of similar size and with an opponent who plays a fixed history-dependent policy. 

\begin{proposition}[POMDP $\subseteq$ MG $+$ opponent playing a general policy]\label{prop:pomdp}
	A POMDP with $S$ hidden states, $A$ actions, $O$ observations, and episode length $H$ can be simulated by a Markov game with opponent playing \textbf{a fixed general policy}, where the Markov game has $OA+O$ states, $A$ actions for the learning agent, $O$ actions for the opponent, and episode length $2H$. 
\end{proposition}

The idea of simulating a POMDP is demonstrated in Figure~\ref{fig:pomdp}: the opponent dictates the next state every two time steps, and since the opponent knows the full trajectory $\{o_1,a_1,o_2,\cdots,o_h,a_h\}$, she can choose $b_h$ according to the conditional distribution of $o_{h+1}$ given $\{o_1,a_1,o_2,\cdots,o_h,a_h\}$ in the POMDP. 
Thus we can simulate the POMDP with a Markov game whose number of states and actions are polynomially related to the original POMDP. We remark that in POMDPs the reward is typically included in the observation so here we do not need to handle it separately. A detailed proof of Proposition \ref{prop:pomdp} is provided in Appendix \ref{app:pomdp}.

\begin{figure}[t]
	\centering
	\includegraphics[width=0.6\columnwidth]{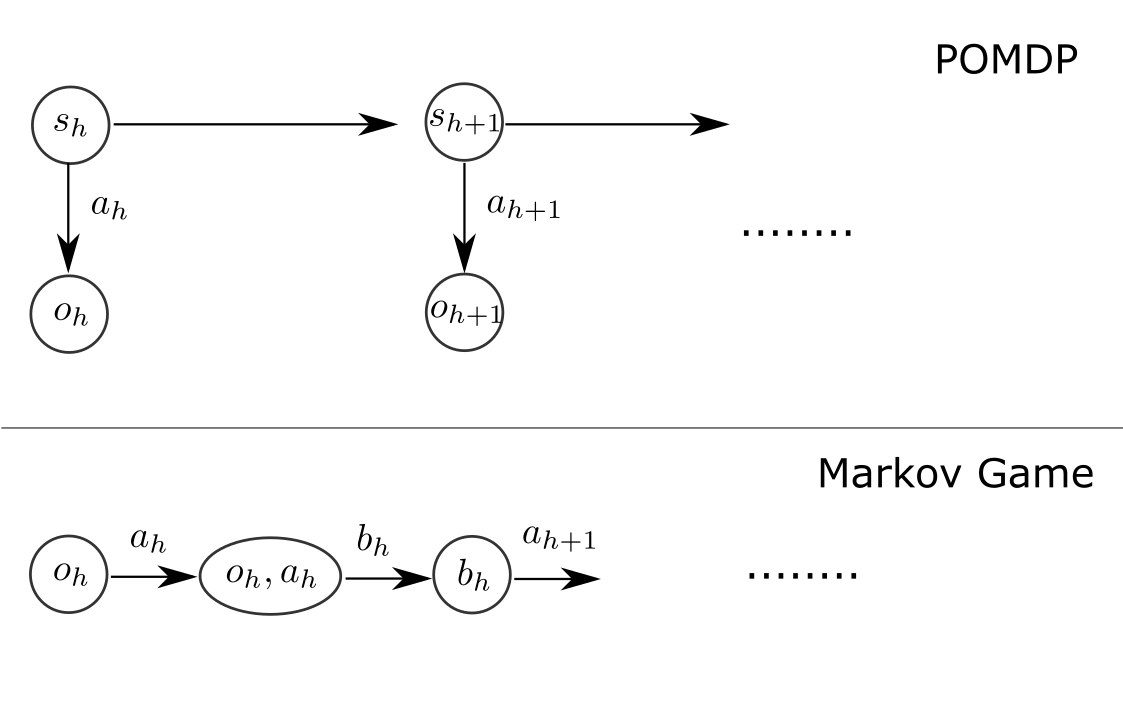}\vspace{-6mm}
	\caption{Simulating a POMDP using a Markov game with a history-dependent opponent.
	The player and the opponent dictates the transition dynamics in turn.  
	 The opponent, which has access to the full history $\{o_1,a_1,o_2,\cdots,o_h,a_h\}$, can always sample her action $b_h$ from $\P[o_{h+1}=\cdot|o_1,a_1,o_2,\cdots,o_h,a_h]$ in the POMDP above, and the next state is exactly equal to $b_h$.}
	\label{fig:pomdp}
\end{figure}

Given that there exists exponential regret lower bound for learning POMDPs \citep[e.g.,][]{jin2020sample}, Proposition \ref{prop:pomdp} immediately implies  no-regret learning in Markov games is  in general intractable if the opponent plays a fixed general policy. 
The proof is a straightforward combination of Proposition \ref{prop:pomdp} and the hard instance constructed in \citet{jin2020sample}, which can be found in Appendix \ref{app:general}.

\subsection{Against opponents playing Markov policies}

Theorem~\ref{thm:action-general} shows that it is statistically hard to compete with the best Markov response to a non-Markov opponent, which is in stark contrast to the case where the opponent plays a fixed Markov policy and the Markov game can be reduced to a single-agent MDP. However, when the opponent is able to choose from a small set of Markov policies, the task of competing with the best Markov policy in hindsight becomes intractable again.

\begin{theorem}
	\label{thm:action-markov}
		There exists a Markov game with $S,A=\cO(H)$ and an opponent who  chooses policy uniformly at random from an unknown set of $H$ Markov policies in each episode, such that  the regret for competing with the best fixed Markov policy in hindsight is  $\Omega(\min\{K, 2^H\}/H)$.
	\end{theorem}

Theorem \ref{thm:action-markov} claims that even restricting the opponent to only play a finite number of Markov policies is insufficient to circumvent the exponential regret lower bound for competing with the best Markov policy in hindsight, as long as the opponent only reveals her actions to the learner.

The proof of Theorem \ref{thm:action-markov} utilizes the following fact that 
we can simulate a latent MDP  by a Markov game of similar size and an opponent who only plays a small set of Markov policies. 

\begin{proposition}[Latent MDP $\subseteq$ MG $+$ opponent playing multiple Markov policies]\label{prop:latent}
	A latent MDP with $L$ latent variables, $S$ states, $A$ actions, and episode length $H$ and binary rewards can be simulated by a Markov game with opponent playing policies chosen from a set of $L$ Markov policies, where the Markov game has $SA+S$ states, $A$ actions for the learning agent, $2S$ actions for the opponent, and episode length $2H$.
\end{proposition}
The proof of Proposition \ref{prop:latent} is deferred to Appendix \ref{app:latent}, which is in a  similar spirit to Proposition \ref{prop:pomdp}. 
Proposition~\ref{prop:pomdp} and \ref{prop:latent}   can be alternatively characterized by the Venn diagram in Figure~\ref{fig:venn}.

Combining Proposition \ref{prop:latent}  with the hardness instance for learning latent MDPs \citep{kwon2021rl}  immediately implies the  exponential lower bound for playing against Markov opponents in Theorem \ref{thm:action-markov}.
A detailed proof is provided in Appendix \ref{app:Markov}.

\begin{figure}[t]
	\centering
	\includegraphics[width=0.6\columnwidth]{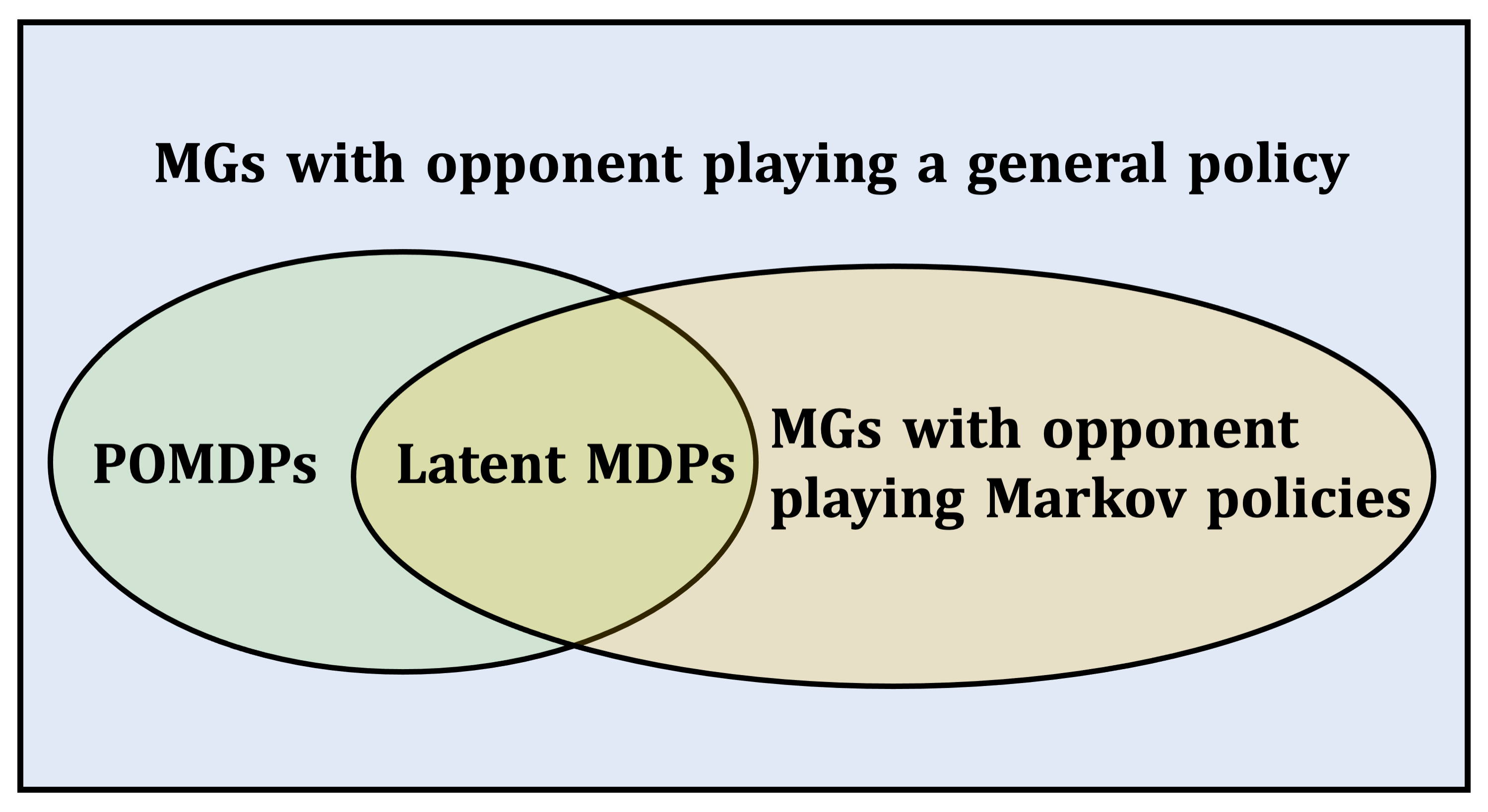}
	\caption{Relation between Markov games (reveal action only), latent MDPs and POMDPs.}
	\label{fig:venn}
\end{figure}


\section{Results for the Revealed-policy Setting}
\label{sec:reveal-policy}

In this section, we study the setting where the opponent reveals the policy she just played to the learner at the end of each episode. 
Formally, in each round of interaction: first the learner and the opponent choose their policies $\mu$ and $\nu$  simultaneously, then an episode is played following $\mu\times\nu$, and after that the learner gets to observe the opponent policy $\nu$.  For this setting, we propose two algorithms with $\sqrt{K}$-regret upper bounds, when either the log-cardinality of the baseline policy class  or the cardinality of the opponent’s policy class is small. This is complemented with an exponential lower bound when neither conditions are true.

\subsection{Finite baseline policy class $\Phi^\star$ }
\label{subsec:finite-baseline}

We first consider the case when the baseline policy class  $\Phi^\star$ to compete with is finite but the opponent's policy class is arbitrary. 
Importantly, we allow both the opponent's polices and the baseline polices to be non-Markov (history-dependent).

\paragraph{Algorithm.} We propose  \OPMD~(Algorithm \ref{alg:md}), which represents \textbf{O}ptimistic \textbf{P}olicy  \textbf{EXP3}, for no-regret learning in this setting.
At a high level, \OPMD~performs any-time EXP3 with optimistic gradient estimate in the baseline policy class  $\Phi^\star$  by viewing each baseline policy as an``action".  Specifically, \OPMD~ maintains a distribution  $\bp$ over the baseline policy class, and in each episode $k$
\begin{itemize}
    \item \textbf{Interaction} (Line \ref{line:1-1}-\ref{line:1-2}). The learner samples a policy $\mu^k$ from  $\Phi^\star$ according to $\bp^k$ and the opponent chooses her policy $\nu^k$ simultaneously. 
    Then a trajectory is  sampled by following  $\mu^k\times\nu^k$. 
    \item \textbf{Optimistic EXP3} (Line \ref{line:1-3}-\ref{line:1-4}). The opponent's policy $\nu^k$ is revealed to the learner, and for every baseline policy $\mu$ in  $\Phi^\star$, the learner computes an optimistic estimate of the value function of $\mu\times\nu^k$ by  using the \textbf{O}ptimistic \textbf{P}olicy \textbf{E}valuation (OPE) subroutine. Then the EXP3 update is incurred with the optimistic value estimates as the negative gradient.  
    \item  \textbf{Model estimate update} (Line \ref{line:1-5}). Using the newly collected data, we update the empirical estimate of the MG model.
\end{itemize} 
In Subroutine \ref{alg:vi}, we formally describe the optimistic  policy evaluation step. In brief, it utilizes the Bellman equation for general policies to perform dynamic programming from step $H$ to step $1$, by using the empirical transition and additionally adding bonus to ensure optimism.

\begin{algorithm}[h!]
    \caption{\textsl{\textbf{O}ptimistic \textbf{P}olicy \textbf{EXP3}}}
 \begin{algorithmic}[1]\label{alg:md}
    \STATE \textbf{input}:  bonus function $\beta:\N\rightarrow\R$, learning rate $(\eta_k)_{k=1}^K$, basesline policy class  $\Phi^\star$ 
    \STATE \textbf{initialize}: initial distribution $\bp^1\in\R^{|\Phi|}$ to be uniform over $\Phi$, visitation counters $N_h(s,\a)=N_h(s,\a,s')=0$ for all $(s,\a,\a',h)$
    \FOR{$k=1,\ldots,K$}
\STATE \alglinelabel{line:1-1} the learner samples $\mu^k\sim\bp^k$ and the adversary chooses $\nu^k$ \emph{simultaneously}
\STATE \alglinelabel{line:1-2} follow $\pi^k=\mu^k\times\nu^k$ to sample  $\{s_h^k,\a_h^k,r_h^k\}_{h=1}^H$\\
  \textsl{{\color{blue} \# optimistic EXP3 }}
\STATE  observe $\nu^k$, and for all $\mu\in\Phi^\star$ compute \alglinelabel{line:1-3}
$\up{V}^{\mu\times\nu^k}_1(s_1)=\text{OPE}(N,\beta,\mu\times\nu^k)$
\STATE then update 
$\bp^{k+1}(\mu)\propto  \exp(\eta_k \cdot\sum_{t=1}^k  \cdot \up{V}^{\mu\times\nu^t}_1(s_1) )$\alglinelabel{line:1-4}\\
\textsl{ {\color{blue} \# update the counters}}
\STATE  for all $h\in[H]$: $N_h(s_h^k,\a_h^k)\leftarrow N_h(s_h^k,\a_h^k)+1$ and $N_h(s_h^k,\a_h^k,s_{h+1}^k) \leftarrow N_h(s_h^k,\a_h^k,s_{h+1}^k)+1$ \alglinelabel{line:1-5}
    \ENDFOR
 \end{algorithmic}
 \end{algorithm}

 \begin{subroutine}[h!]
     \caption{\textsl{\textbf{O}ptimistic  \textbf{P}olicy \textbf{E}valuation~$(N,\beta,\pi)$}}
     \begin{algorithmic}\label{alg:vi}
         \STATE initialize $V_{H+1}(\tau_{H+1})= 0$ for all $\tau_{H+1}$
\FOR{$(s,\a,h,s')\in\cS\times\cA\times[H]\times\cS$}
\STATE 
 $$
\widehat\P_h(s'\mid s,\a)=\begin{cases}
{N_{h}(s,\a,s')}/{ N_h(s,\a)}, &\text{ if } N_h(s,\a)\neq0\\
    {1}/{S}, & \qquad \text{otherwise}
 \end{cases}
 $$ 
\ENDFOR         
\FOR{$h=H,\ldots,1$}
         \FOR{all $\tau_h=(s_1,\a_1,\ldots,s_h)\in(\cS\times\cA)^{h-1}\times\cS$}
         \STATE $Q_{h}(\tau_h,\a) = \E_{s'\sim\widehat\P_h(\cdot\mid s_h,\a)}\left[V_{h+1}([\tau_h,\a,s'])\right]+r_h(s_h,\a)+\beta(N_{h}(s_h,\a))$
         \STATE $Q_{h}(\tau_h,\a) =\min\left\{Q_{h}(\tau_h,\a),H-h+1\right\}$
         \STATE $V_h(\tau_h)= \E_{\a\sim\pi(\cdot\mid \tau_h)} [Q_{h}(\tau_h,\a)]$
         \ENDFOR
         \ENDFOR
\STATE \textbf{return} $V_1(s_1)$
     \end{algorithmic}
 \end{subroutine}

 \paragraph{Theoretical guarantee.}
 Below we present the main theoretical guarantee for \OPMD.
 \begin{theorem}\label{thm:up-1}
Let $c$ be a large absolute constant. 
In Algorithm \ref{alg:md}, choose 
$       
\eta_k=\sqrt{{\log| \Phi^\star|}/{(kH^2)}}$ and   $\beta(n)=\sqrt{{H^2S\iota}/{\max\{n,1\}}}$ 
where 
$\iota=c\log(SAHK/\delta)$. Then with probability at least $1-\delta$, for all $k\in[K]$
$$
\Reg_{\Phi^\star}(k)\le \cO\left( \sqrt{kH^2\log|\Phi^\star|}+\sqrt{kS^2AH^4\iota^2}\right).
$$
\end{theorem}
Theorem \ref{thm:up-1} claims that \OPMD~with standard UCB-bonus achieves $\cO(\sqrt{k})$-regret with high probability, when competing with the best policy in hindsight in the baseline class.
Notably, the regret  only  depends logarithmically on the  cardinality of the baseline class and is  independent of the opponent's policy class. In particular, if we choose the baseline policy class to be the collections of all deterministic\footnote{Competing against all Markov policies is equivalent to competing with all deterministic Markov policies because for any general policy there always exists a Markov best-response that is also deterministic.} Markov policies ($|\Phi^\star|=A^{SH}$), then Theorem \ref{thm:up-1} immediately implies  $\cO\left( \sqrt{kS^2AH^4\iota^2}\right)$ regret upper bound for competing with the best Markov policy in hindsight.
Moreover, it further implies the same regret upper bound for competing against the value of Nash equilibria, i.e., the regret in equation \eqref{eq:pre-reg}, because there always  exists a Markov Nash equilibrium. The proof of Theorem \ref{thm:up-1} can be found in Appendix \ref{pf:up-1}.

\subsection{Finite unknown opponent policy class $\Psi^\star$}
In Section \ref{subsec:finite-baseline}, we study the problem of competing with a finite baseline policy while allowing arbitrary opponent polices. 
In this subsection, we turn to a complementary setting where the baseline policy class consists of \emph{all} the general polices while the opponent policy class $\Psi^\star$ is finite but \emph{unknown}.

\paragraph{Algorithm.} Based on \OPMD, we propose \textsf{Adaptive OP-EXP3} (Algorithm \ref{alg:md-restart}), which  represents \textbf{adaptive} \textbf{O}ptimistic \textbf{P}olicy \textbf{EXP3}. Compared to its prototype, \AOPMD~incorporates the following two key modifications
\begin{itemize}
    \item \textbf{Lazy model update} (Line \ref{line:2-6}-\ref{line:2-7}). \textsf{Adaptive OP-EXP3} maintains two empirical model estimates: the latest version and a lazy version that are computed by using counter $N$ and $\NL$ respectively. Counter $N$ is  promptly updated in each episode as in  \OPMD, while counter $\NL$ copies the values in $N$ each time a state-action counter in $N$ is doubled or a new opponent policy is observed. Importantly, \AOPMD~always uses the lazy model estimate for optimistic policy evaluation (Line 6).
    \item \textbf{Adaptive player policy class} (Line \ref{line:2-6}-\ref{line:2-9}). Each time the opponent reveals a new policy (i.e., a policy not in the historical opponent policy set $\Psi^k$) or the lazy model is updated, the learner recomputes its policy class $\Phi$ to include the optimistic best responses to all the possible mixtures of historical opponent policies. After that, EXP3 is restarted from the uniform distribution over $\Phi$.
\end{itemize}
We formally describe how to recompute the player policy class in Subroutine \ref{alg:br} where we in fact only consider an $\epsilon$-cover of all the possible mixtures of historical opponent policies. And for each such mixture, we compute an optimistic best response, by invoking the optimistic policy evaluation subroutine with the lazy model estimate.

Intuitively, the reason for only including the best responses to policy mixtures in the player policy class   is that the best general policy in hindsight is always a best response to a mixture of the historical opponent policies.
Moreover, by doing so, we effectively shrink the log-cardinality of the baseline policy class to $\tilde\cO(|\Psi^\star|)$ that is the size of the opponent policy class, while still remaining competitive with any general policy.

\begin{algorithm}[h]
    \caption{\textsl{\textbf{Adaptive} \textbf{O}ptimistic \textbf{P}olicy \textbf{EXP3}}}
 \begin{algorithmic}[1]\label{alg:md-restart}
    \STATE \textbf{input}:   bonus function $\beta:\N\rightarrow\R$, learning rate $(\eta_k)_{k=1}^K$, grid resolution $\epsilon$.
    \STATE \textbf{initialize}: baseline policy class $\Phi$ and distribution $\bp^1\in\R^{|\Phi|}$ arbitrarily, visitation counters $N_h(s,\a)=N_h(s,\a,s')=\NL_h(s,\a)=\NL_h(s,\a,s')=0$ for all $(s,\a,\a',h)$, $\Psi^1=\emptyset$, $m^1=0$
    \FOR{$k=1,\ldots,K$}
\STATE the learner samples $\mu^k\sim\bp^k$ and the adversary chooses $\nu^k$ \emph{simultaneously} \alglinelabel{line:2-1}
\STATE follow $\pi^k=\mu^k\times\nu^k$ to sample  $\{s_h^k,\a_h^k,r_h^k\}_{h=1}^H$\alglinelabel{line:2-2}\\
\textsl{{\color{blue} \# optimistic EXP3 }}
\STATE observe $\nu^k$, and for all $\mu\in\Phi$ compute  \alglinelabel{line:2-3} 
$\up{V}^{\mu\times\nu^k}_1(s_1)=\text{OPE}(\NL,\beta,\mu\times\nu^k)$
 \STATE then update  \alglinelabel{line:2-4} 
$\bp^{k+1}(\mu)\propto  \exp(\eta_{k} \cdot  \sum_{t=m^k+1}^k \up{V}^{\mu\times\nu^t}_1(s_1 ))$\\
 \textsl{ {\color{blue}\# update the counters}}
\STATE  for all $h\in[H]$: $N_h(s_h^k,\a_h^k)\leftarrow N_h(s_h^k,\a_h^k)+1$ and $N_h(s_h^k,\a_h^k,s_{h+1}^k) \leftarrow N_h(s_h^k,\a_h^k,s_{h+1}^k)+1$ \alglinelabel{line:2-5}\\
 \textsl{{\color{blue} \# update the lazy model and policy class }}
\IF{$\nu^k\notin \Psi^k$\textbf{or} $\exists h$ s.t. $N_h(s_h^k,\a_h^k) \ge 2\NL_h(s_h^k,\a_h^k)$}\alglinelabel{line:2-6}
\STATE $\NL\leftarrow N$, $\Psi^{k+1} \leftarrow \Psi^k \cup\{\nu^k\}$, $m^{k+1}\leftarrow k$\alglinelabel{line:2-7}
\STATE $\Phi\leftarrow\text{OBR}(\NL,\beta,\Psi^{k+1},\epsilon)$\alglinelabel{line:2-8}
\STATE reset $\bp^{k+1}$ to be uniform over $\Phi$ \alglinelabel{line:2-9}
\ELSE
\STATE {$\Psi^{k+1} \leftarrow \Psi^k $ and $m^{k+1}\leftarrow m^k$ }\alglinelabel{line:2-10}
\ENDIF
    \ENDFOR
 \end{algorithmic}
 \end{algorithm}

\begin{subroutine}[h]
    \caption{\textsl{\textbf{O}ptimistic \textbf{B}est \textbf{R}esponse $(\NL,\beta,\Psi,\epsilon)$}}
    \begin{algorithmic}\label{alg:br}
       \STATE \textbf{initialize:} $\BR=\{\}$
        \STATE denote the polices in $\Psi$ by $\nu^{(1)},\ldots,\nu^{(|\Psi|)}$
        \STATE denote by $\Delta_{|\Psi|}^\epsilon$ an $\epsilon$-cover of $\Delta_{|\Psi|}$ w.r.t. $\ell_1$-norm
        \FOR{$w \in\Delta_{|\Psi|}^\epsilon$}
        \STATE Select an arbitrary \\
        $\quad \mu\in\argmax_{\hat\mu} \sum_{i=1}^{|\Psi|} w_i \times \text{OPE}(\NL,\beta,\hat\mu\times\nu^{(i)})$
        \STATE $\BR \leftarrow \BR \cup \{\mu\}$
        \ENDFOR
        \STATE return $\BR$
    \end{algorithmic}
\end{subroutine}

 \paragraph{Theoretical guarantee.} 
Now we present the theoretical guarantee for \AOPMD, under the following adaptive learning rate schedule 
\begin{equation}\label{eq:lr}
\eta_k=\sqrt{\frac{|\Psi^k|\log(K)}{(k-m^k)H^2}},
\end{equation}
 where $\Psi^k$ contains all the different policies the opponent has played before the $k^{\rm th}$ episode, and $m^k$ denotes the index of the most recent episode when EXP3 is restarted (Line 11) before the $k^{\rm th}$ episode.
 \begin{theorem}\label{thm:up-2}
    Let $c$ be a large absolute constant. 
    In Algorithm \ref{alg:md-restart}, choose  the learning rate adaptively by \eqref{eq:lr}, 
    $\epsilon=1/K$ and $\beta(n)=\sqrt{{H^2S\iota}/{\max\{n,1\}}}$
     where $\iota=c\log(SAHK/\delta)$. Then with probability at least $1-\delta$, for all $k\in[K]$
     $$
     \Reg(k)\le \cO\left( \sqrt{k\left(S^2AH^2+|\Psi^k|SAH+|\Psi^k|^2\right)H^2\iota^2} \right).
     $$
\end{theorem}
Given that the opponent only plays policies from a finite class $\Psi^\star$, Theorem \ref{thm:up-2} guarantees that \AOPMD~ suffers regret at most 
$$\cO\left( \sqrt{k\left(S^2AH^2+|\Psi^\star|SAH+|\Psi^\star|^2\right)H^2\iota^2}\right)$$
in competing with the best \emph{general} policy in hindsight. 
Moreover, note that the bound in Theorem \ref{thm:up-2} depends linearly on the number of different historical opponent policies. As a result, the regret of $\AOPMD$ is still sublinear even if the opponent policy class keeps  expanding as $k$ increases, as long as its cardinality is order $\small{o}(\sqrt{k})$.
 The proof of Theorem \ref{thm:up-2} can be found in Appendix \ref{pf:up-2}.

\subsection{Statistical hardness with large $\Phi^\star$ and $\Psi^*$}
Theorem~\ref{thm:up-1} and~\ref{thm:up-2} show that when either $\log|\Phi^\star|$ (the log-cardinality of the baseline policy class) or $|\Psi^*|$ (the cardinality of the opponent's policy class) is ploynomial, a sublinear regret bound is obtainable. We now complement these two results with a lower bound when both conditions are violated, i.e., when the size of $\Phi^\star$ is doubly exponential and the size of $\Psi^*$ is exponential.

\begin{theorem}
	\label{thm:general-lower}
	There exists a Markov game with $S=1$, $|\cA_{\max}|=|\cA_{\min}|=2$, $|\Psi^*|=2^H$ such that the regret for competing with the best general policy in hindsight is $\Omega(\min\{K, 2^H\})$, even if the adversary reveals her policy after each episode.
\end{theorem}

The construction for this lower bound is quite simple. Consider a Markov game with horizon $H$ and only $1$ state. The agent only receives non-zero reward if at the final time step, it plays the same action as the opponent, \rm{i.e.},  $r_H(s,(a,b))=\mathbf{1}[a=b]$. Now, suppose that in each episode, the opponent samples randomly from the set of all deterministic Markov policies; any algorithm would have an expected value of ${1}/{2}$, as $b_H\sim {\rm Ber}(1/2)$. However, the best history-dependent policy in hindsight would be able to predict $b_H$ by memorizing $b_1,\cdots,b_{H-1}$ when the number of episodes is not exponentially large. This gives the claimed  $\Omega(\min\{K, 2^H\})$ lower bound. A formal proof can be found in Appendix \ref{pf:general-lower}.

\section{Computational Hardness}
\label{sec:comp}

Finally, we provide a computational lower bound for this problem. We remark that this lower bound holds even if 
(a) the transitions of the Markov game are known, (b) the opponent reveals the policy she just played at the end of each episode, and (c) the opponent can only choose from a small known set of Markov policies ($|\Psi^*|=\cO(H)$). 
Therefore, the lower bound applies to all the settings considered in this paper.

\begin{theorem}
	\label{thm:computation-lower}
	If an algorithm achieves ${\rm poly}(S,A,H)\cdot K^{1-c}$ expected regret with a constant $c>0$ in the setting that satisfies the above condition $(a),(b),(c)$, then its computational complexity cannot be ${\rm poly}(S,A,H,K)$ unless ${\rm NP\subseteq BPP}$.\footnote{BPP is the probabilistic version of P, and ${\rm NP\subseteq BPP}$ is believed to be highly unlikely in computational complexity literature.}
\end{theorem}

This computational lower bound suggests that the best we can hope for is a statistically efficient but computationally intensive algorithm. It also renders statistically efficient value-iteration or Q-learning style algorithms for this problem unlikely, unless they employ NP-hard subroutines.

The proof of Theorem~\ref{thm:computation-lower} depends on the construction in Proposition $6$ of~\citet{steimle2021multi}, which reduces solving 3-SAT to finding the best Markov policy in a latent MDP. 
We provide a full proof in Appendix~\ref{pf:comp}.


\section{Conclusion}

This paper studies no-regret learning  of Markov games with adversarial opponents.
We provide a complete set of positive and negative results for competing with the best fixed policy in hindsight.
In the standard setting where only the actions of opponents are revealed, we prove it is statistically intractable to compete with the best fixed Markov policy in hindsight, even if  the opponent only chooses from a limited number of Markov policies.
In the revealed-policy setting, we propose new algorithms with $\sqrt{K}$-regret bound when either the log-cardinality of the baseline policy class  or the cardinality of the opponent’s policy class is small. 
Additionally, an exponential lower bound is derived  when both quantities are large.
Finally, we turn to the computational efficiency and prove achieving sublinear regret is in general computationally hard even in the very benign scenario.

\section*{Acknowledge}
We thank Zhuoran Yang for valuable discussions. 

\bibliographystyle{plainnat}
\bibliography{ref}

\newpage
\appendix

\section{Proofs for Section \ref{sec:action-hardness}}

\subsection{Proof of Theorem \ref{thm:action-general}}
\label{app:general}
    Because we can simulate any POMDP with a MG by using Proposition \ref{prop:pomdp}, it suffices to show there exists a hard POMDP instance with $\cO(1)$ number of states, actions and observations so that any algorithm will suffer $\Omega(\min\{4^H,K\})$ regret 
    when competing with the optimal Markov policy of this POMDP.
    
    We use the hard instance constructed in \citet{jin2020sample}. There are two  states $s_g,s_b$ and four actions. There is special action sequence $a^\star_1,\ldots,a^\star_{H-1}$ sampled independently and uniformly at random from the action set, which is unknown to the learner. The transition dynamics are constructed so that (a) the agent always starts in  $s_g$ at step $1$, (b) at each step $h$ the agent will transition to $s_g$ if and only if she is currently in $s_g$ and plays the special action $a^\star_h$, and otherwise will go to $s_b$. 
    At the first $H-1$ steps, the two states emit the \emph{same} observation that contains reward $0$.
     At step $H$, $s_g$ emits reward $1$ while $s_b$ still emits a zero-reward observation. 
     It is straightforward to see the optimal policy is to play the special action sequence, which is \emph{Markov}. However, because $s_g$ and $s_b$ are totally indistinguishable from observations at the first $H-1$ steps, finding this action sequence will cost at least $\Omega(4^H)$ episodes in general, which implies a $\Omega(\min\{4^H,K\})$ regret lower bound for competing with the optimal Markov policy.

\subsection{Proof of Proposition \ref{prop:pomdp}}
\label{app:pomdp}

We describe how to simulate a POMDP with a Markov game and an opponent playing a fixed general policy. 

Each step in the POMDP is simulated by two consecutive steps in the Markov game, and the transition dynamics of the Markov game have the following special structures:
\begin{itemize}
    \item At an even step, the transition only depends on the action of the opponent. 
    Moreover, the next state is always equal to the opponent's action regardless of  the current state.
    \item At an odd step, the transition only depends on the action of the learner, and the next state is simply an augmentation of the current state and the learner's action. 
\end{itemize}

Specifically, suppose in the POMDP, at step $h$, the learner starts with history $o_1,a_1,\ldots,o_h$ and plays action $a_h$, then observes $o_{h+1}$ sampled from $\P(o_{h+1}=\cdot \mid o_1,a_1,\ldots,o_h,a_h)$. 
In this case, the corresponding two steps in the POMG will be: at step $2h-1$, the learner starts at state $o_h$ and takes action $a_h$, then the environment transitions to state $(o_h,a_h)$; at step $2h$, the opponent starts at state $(o_h,a_h)$ and takes action $o_{h+1}$ sampled from  $\P(o_{h+1}=\cdot \mid o_1,a_1,\ldots,o_h,a_h)$, then the environment transitions to $o_{h+1}$ that is exactly equal to the action of the opponent. Note that here the opponent is playing a history-dependent policy. 

It is direct to see there are $O A+ O$ distinct states, $A$ actions for the learner and $O$ actions for the opponent in this Markov game. Besides, the episode length is $2H$.

\subsection{Proof of Theorem \ref{thm:action-markov}}\label{app:Markov}
    By Proposition \ref{prop:latent},  we can simulate any latent MDP with a MG.
    As a result, it suffices to show there exists a hard latent MDP with $\cO(1)$ states, $\cO(H)$ actions and $H$ latent variables  so that any algorithm will suffer $\Omega(\min\{4^H,K\}/H)$ regret when competing with the optimal Markov policy of this latent MDP.
    
    We utilize the hard latent MDP instance constructed in Theorem 3.1~\citep{kwon2021rl}.\footnote{Despite  \citet{kwon2021rl} study the stationary setting, their constructions can be trivially adapted to handle the nonstationary setting and gives a stronger lower bound which is the one we state here.}
     In the latent MDP instance, there is a collection of $H$ unknown MDPs, each of which has $\cO(1)$ states, $\cO(H)$ actions and binary rewards. 
    At the beginning of each episode the environment \emph{secretly} draws an  MDP uniformly at random from these $H$ MDPs, and then the algorithm interacts with this MDP without knowing which one it is. \citet{kwon2021rl} prove that it takes $\Omega(4^H)$ episodes to learn a policy that is $\cO({1}/{H})$-optimal compared to the best Markov policy, where the optimality is defined using the average value over the $H$ MDPs. 
    By the standard  online-to-batch conversion \citep[e.g.,][]{lattimore2020bandit}, it immediately implies a $\Omega(\min\{4^H,K\}/H)$ regret lower bound for competing with the optimal Markov policy.

\subsection{Proof of Proposition \ref{prop:latent}}
\label{app:latent}

To begin with, we recall the definition of latent MDPs \citep{kwon2021rl}.
At the beginning of each episode the environment \emph{secretly} draws an  MDP uniformly at random from $L$ unknown MDPs, then the algorithm interacts with this MDP without knowing which one it is. 

Denote by $q\in\Delta_L$ the latent distribution over these $L$ MDPs and $\P^i_h(s'|s,a)$ ($r_h^i(s,a)$) the transition (reward) function of the $i^{\rm th}$ MDP. 
In each episode of the Markov game
\begin{itemize}
	\item The opponent secretly samples  $t\sim q$ before step $1$, and keeps it hidden from the learner throughout.
	\item At step $2h-1$, the transitions are deterministic, and only depend on the current state and the learner's action. Specifically, the environment will transition to an augmenting state $(s,a)$ if the learner takes action $a$ at state $s$ regardless of what action the opponent picks. There is no reward at this step.
	\item At step $2h$, the transitions and rewards are still deterministic, but only depend on the opponent's action. Formally, the environment will transition to state $s'$ from an augmenting state $(s,a)$ and the learner will receive reward $r'\in\{0,1\}$, if the opponent takes action $(s',r')$, the probability of which is given by\vspace{-2mm}
	$$\nu_{2h}\left((s',r')|(s,a),t\right)=\P_{h}^t(s'| s,a)\times \mathbf{1}\left(r_h^t(s,a)=r'\right).$$
\end{itemize}

It is direct to see interacting with this MG is exactly equivalent to interacting with the original latent MDP. In particular, there is \emph{no additional information} revealed in the MG because the opponent's action is always equal to the next state and the reward.

\section{Proofs for Section \ref{sec:reveal-policy}}

\subsection{Proof of Theorem \ref{thm:up-1} }
\label{pf:up-1}
We first introduce several notations that will be frequently used in our proof.
 Let $\tau_h = [s_1,\a_1,\ldots,s_{h-1},\a_{h-1},s_h]$.
 Denote by $N^k$ the collection of counters at the \emph{beginning} of episode $k$. Denote by $\hat\P^k$ the empirical transition computed by using $N^k$, i.e., for any $(s,\a,h,s')\in\cS\times\cA\times[H]\times\cS$
\begin{align*}
\widehat\P_h^k(s'\mid s,\a) = 
\begin{cases}
     &\frac{N_{h}^k(s,\a,s')}{N_h^k(s,\a)}\qquad \text{ if } N_h(s,\a)\neq0, \\
     &1/S\qquad\qquad\  \text{otherwise}.
\end{cases}
\end{align*}
Given an arbitrary policy $\pi$, we define $\up{V}^{\pi,k}$ ($\up{Q}^{\pi,k}$) that is the optimistic estimate of $V^\pi$ ($Q^\pi$) as following: for any $h\in[H]$,
\begin{equation}
\begin{cases}
\up V_h^{k,\pi}(\tau_h) = \E_{\a\sim\pi(\cdot\mid\tau_{h})} \left[\up Q_h^{k,\pi}(\tau_h,\a) \right],\\
\up Q_h^{k,\pi}(\tau_h,\a) =  \min\left\{ \E_{s' \sim\hat\P^k_h(\cdot \mid s_{h},\a)} \left[\up V_{h+1}^{k,\pi}([\tau_h,\a,s']) \right] +r_h(s_h,\a)+ \beta(N_h^k(s_h,\a)), H-h+1\right\},
\end{cases}
\end{equation}
and we define $\up V_{H+1}^{k,\pi} \equiv 0$. We comment that by definition $\up{V}_1^{k,\pi}(s_1)=\text{UCB-VI}(N^k,\beta,\pi)$ for all $k,\pi$.

For the purpose of proof, we further introduce the following auxiliary function for controlling the optimism of $\up{V}^{\pi,k}$ ($\up{Q}^{\pi,k}$) against the true value function $V^\pi$ ($Q^\pi$): for any $h\in[H]$
\begin{equation}
\begin{cases}
\tilde V_h^{k,\pi}(\tau_h) = \E_{\a\sim\pi(\cdot\mid\tau_h)} \left[\tilde Q_h^{k,\pi}(\tau_h,\a) \right],\\
\tilde Q_h^{k,\pi}(\tau_h,\a) =  \min\left\{ \E_{s' \sim\P_h(\cdot \mid s_{h},\a)} \left[\tilde V_{h+1}^{k,\pi}([\tau_h,\a,s']) \right] +2\beta(N^k_h(s_h,\a)), H-h+1\right\},
\end{cases}
\end{equation}
and we define $\tilde V_{H+1}^{k,\pi} \equiv 0$. 
Compared to $\up{V}^{k,\pi}$, $\tilde V_h^{k,\pi}$ is defined using the groundtruth transition function $\P$, it does not contain the reward function and the bonus function is doubled.

Finally, recall we choose the bonus function to be 
$$
\beta(t) = H\sqrt{\frac{S\iota}{\max\{t,1\}}},
$$
where  $\iota =\log(KHSA/\delta)$ with $c$ being some large absolute constant.

\begin{lemma}[Optimism]\label{lem:optimisL-1}
    With probability at least $1-\delta$, for all $(k,h)\in[K]\times[H+1]$ and all general policy $\pi$, 
    $$
    0 \le \up V^{k,\pi}_{h}(\tau_h) - V^{\pi}_{h}(\tau_h) \le \tilde V^{k,\pi}_{h}(\tau_h) \quad \text{for all }\tau_h.
    $$
    \end{lemma}
    \begin{proof}[Proof of Lemma \ref{lem:optimisL-1}]
    To begin with, by the Azuma-Hoeffding inequality and standard union bound argument, we have that with probability at least $1-\delta$:
    $$
\| \hat\P^k_h(\cdot\mid s,\a)- \P_h(\cdot\mid s,\a)\|_1 \le \frac{1}{H} \beta(N_h^k(s,\a)) \quad \text{ for all }\ (s,\a,h,k) \in\cS\times\cA\times[H]\times[K]. 
    $$
    Below, we prove the lemma conditioning on the event above being true. We prove the lemma by induction and start with the upper bound. The inequality holds for step $H+1$ trivially because 
    $V^{k,\pi}_{H+1} = V^{\pi}_{H+1} =  \tilde V^{k,\pi}_{H+1} = 0$. Assume the inequality holds for step $h+1$. At step $h$, notice that 
    \begin{align*}
        \tilde V_h^{k,\pi}(\tau_h) = & \E_{\a\sim\pi(\cdot\mid\tau_h)} \left[\tilde Q_h^{k,\pi}(\tau_h,\a) \right],\\
        \up V^{k,\pi}_{h}(\tau_h) - V^{\pi}_{h}(\tau_h) 
        = & \E_{\a\sim\pi(\cdot\mid\tau_h)} \left[\up Q^{k,\pi}_{h}(\tau_h,\a) - Q^{\pi}_{h}(\tau_h,\a)\right].
    \end{align*}
    Therefore, it suffices to prove 
    $$
    \up Q^{k,\pi}_{h}(\tau_h,\a) - Q^{\pi}_{h}(\tau_h,\a) \le \tilde Q_h^{k,\pi}(\tau_h,\a) \quad \text{for all }\tau_h,\a,
    $$
    which follows from
    \begin{align*}
        & \up Q^{k,\pi}_{h}(\tau_h,\a) - Q^{\pi}_{h}(\tau_h,\a) \\
    \le & \min\left \{   \E_{s' \sim\hat\P_h^k(\cdot \mid s_{h},\a)} \left[\up V_{h+1}^{k,\pi}([\tau_h,\a,s']) \right]- \E_{s' \sim\P_h(\cdot \mid s_{h},\a)} \left[ V_{h+1}^{\pi}([\tau_h,\a,s']) \right]+\beta(N_h^k(s_h,\a)), H-h+1     \right\}\\
    = & \min\bigg \{   \E_{s' \sim\hat\P_h^k(\cdot \mid s_{h},\a)} \left[\up V_{h+1}^{k,\pi}([\tau_h,\a,s']) \right]- \E_{s' \sim\P_h(\cdot \mid s_{h},\a)} \left[\up V_{h+1}^{k,\pi}([\tau_h,\a,s']) \right] \\
    & \qquad + \E_{s' \sim\P_h(\cdot \mid s_{h},\a)} \left[\up V_{h+1}^{k,\pi}([\tau_h,\a,s']) - V_{h+1}^{\pi}([\tau_h,\a,s']) \right]+\beta(N_h^k(s_h,\a)), H-h+1        \bigg\}\\
    \le & \min\left\{ 2\beta(N_h^k(s_h,\a))+ \E_{s' \sim\P_h(\cdot \mid s_{h},\a)} \left[\tilde V_{h+1}^{k,\pi}([\tau_h,\a,s']) \right], H-h+1     \right\} =\tilde Q_h^{k,\pi}(\tau_h,\a), 
    \end{align*}
    where the last inequality follows from  the induction hypothesis and $\| \hat\P^k_h(\cdot\mid s_h,\a)- \P_h(\cdot\mid s_h,\a)\|_1 \le  \beta(N_h^k(s_h,\a))/H$. 

    Similarly, for the lower bound, we only need to show 
    $$
    \up Q^{k,\pi}_{h}(\tau_h,\a) \ge  Q^{\pi}_{h}(\tau_h,\a)  \quad \text{for all }\tau_h,\a,
    $$
    which follows similarly from
    \begin{align*}
        & \up Q^{k,\pi}_{h}(\tau_h,\a) - Q^{\pi}_{h}(\tau_h,\a) \\
    \ge & \min\left \{   \E_{s' \sim\hat\P_h^k(\cdot \mid s_{h},\a)} \left[\up V_{h+1}^{k,\pi}([\tau_h,\a,s']) \right]- \E_{s' \sim\P_h(\cdot \mid s_{h},\a)} \left[ V_{h+1}^{\pi}([\tau_h,\a,s']) \right]+\beta(N_h^k(s_h,\a)), 0   \right\}\\
    = & \min\bigg \{   \E_{s' \sim\hat\P_h^k(\cdot \mid s_{h},\a)} \left[\up V_{h+1}^{k,\pi}([\tau_h,\a,s']) \right]- \E_{s' \sim\P_h(\cdot \mid s_{h},\a)} \left[\up V_{h+1}^{k,\pi}([\tau_h,\a,s']) \right] \\
    & \qquad + \E_{s' \sim\P_h(\cdot \mid s_{h},\a)} \left[\up V_{h+1}^{k,\pi}([\tau_h,\a,s']) - V_{h+1}^{\pi}([\tau_h,\a,s']) \right]+\beta(N_h^k(s_h,\a)), 0      \bigg\}\\
    \ge & \min\bigg \{   -\beta(N_h^k(s_h,\a)) + \E_{s' \sim\P_h(\cdot \mid s_{h},\a)} \left[\up V_{h+1}^{k,\pi}([\tau_h,\a,s']) - V_{h+1}^{\pi}([\tau_h,\a,s']) \right]+\beta(N_h^k(s_h,\a)),0 \bigg\}\\
    \ge & 0, 
    \end{align*}
    where the last inequality follows from  the induction hypothesis and the second last one uses $\| \hat\P^k_h(\cdot\mid s_h,\a)- \P_h(\cdot\mid s_h,\a)\|_1 \le  \beta(N_h^k(s_h,\a))/H$. 
    \end{proof}

\begin{proof}[Proof of Theorem \ref{thm:up-1}]
    In the remainder of this section, we show how to control $\Reg(K)$. The upper bound for $\Reg(k)$ ($k\in[K]$) can be derived by repeating precisely the same arguments.

For simplicity of notations, denote $\pi^k = \mu^k\times \nu^k$.
   By the optimism of $\up{V}$ (Lemma \ref{lem:optimisL-1}), with probability at least $1-\delta$,
    \begin{align*}
        &\max_{\mu^\star}\sum_{k=1}^K V^{\mu^\star\times \nu^k}_1(s_1) - \sum_{k=1}^K V^{\pi^k}_1(s_1)\\
        \le   & \left(\max_{\mu^\star} \sum_{k=1}^K \up V^{\mu^\star\times\nu^k,k}_1(s_1) - \sum_{k=1}^K   \mathbb{E}_{\mu\sim\bp^k} [\up{V}^{\mu\times\nu^k,k}_1(s_1) ]\right)
        +
        \left( \sum_{k=1}^K   \mathbb{E}_{\mu\sim\bp^k} [\up{V}^{\mu\times\nu^k,k}_1(s_1) ]  - \sum_{k=1}^KV^{\pi^k}_1(s_1)\right).
    \end{align*}
    The first term is upper bounded by the regret bound of anytime EXP3, which is of order $\mathcal O(H\sqrt{\log(|\Phi^\star|)K})$ \citep[e.g.,][]{lattimore2020bandit}. Below, we focus on controlling the second term. Since $\mu^k\sim \bp^k$, by the Azuma-Hoeffding inequality, with probability at least $1-\delta$,
    \begin{align*}
        \sum_{k=1}^K   \mathbb{E}_{\mu\sim\bp^k} [\up{V}^{\mu\times\nu^k,k}_1(s_1) ]  - \sum_{k=1}^KV^{\pi^k}_1(s_1)
        &\le \sum_{k=1}^K   \up{V}^{\mu^k\times\nu^k,k}_1(s_1)  - \sum_{k=1}^KV^{\pi^k}_1(s_1)+\cO(H\sqrt{K\log(1/\delta)}).
    \end{align*}
   By Lemma \ref{lem:optimisL-1} and the Azuma-Hoeffding inequality, with probability at least $1-2\delta$, 
   \begin{align*}
     \sum_{k=1}^K   \up{V}^{\mu^k\times\nu^k,k}_1(s_1)  - \sum_{k=1}^KV^{\pi^k}_1(s_1) 
    \le &\sum_{k=1}^K   \tilde{V}^{\mu^k\times\nu^k,k}_1(s_1) \\
    \le & \sum_{k=1}^K  \sum_{h=1}^H \E_{(s_h,a_h)\sim\pi^t}\left[2\beta(N^k_h({s_h,a_h}))\right]\\
    \le &2 \sum_{h=1}^H  \sum_{k=1}^K  \beta(N^k_h({s_h^k,a_h^k}))+\cO(H^2\sqrt{KS\iota^2})\\
    \le & \cO\left( \sqrt{KS^2AH^4\iota^2}\right),
   \end{align*}
   where the final inequality follows from the definition of $\beta$ and the standard pigeon-hole argument.

Combining all the relations above, taking a union bound and rescaling $\delta$ complete the proof.
\end{proof}

\subsection{Proof of Theorem \ref{thm:up-2} }
\label{pf:up-2}
 At the very beginning of the proof of Theorem \ref{thm:up-1}, we define several useful quantities $\hat\P^k,\up{V}^k,\tilde{V}^k$ using the regular counter $N^k$.
In this section, with slight abuse of notations, we change their definitions by replacing $N^k$ with $N^{k,{\rm lazy}}$ that is the collection of the \emph{lazy} counters at the \emph{beginning} of episode $k$. Formally, denote by $\hat\P^k$ the empirical transition computed by using $N^{k,{\rm lazy}}$, i.e., for any $(s,\a,h,s')\in\cS\times\cA\times[H]\times\cS$
\begin{align*}
\widehat\P_h^k(s'\mid s,\a) = 
\begin{cases}
     &\frac{N^{k,{\rm lazy}}_h(s,\a,s')}{N^{k,{\rm lazy}}_h(s,\a)}\qquad \text{ if } N_h(s,\a)\neq0, \\
     &1/S\qquad\qquad\  \text{otherwise}.
\end{cases}
\end{align*}
Given an arbitrary policy $\pi$, we define $\up{V}^{\pi,k}$ ($\up{Q}^{\pi,k}$) that is the optimistic estimate of $V^\pi$ ($Q^\pi$) as following: for any $h\in[H]$,
\begin{equation}
\begin{cases}
\up V_h^{k,\pi}(\tau_h) = \E_{\a\sim\pi(\cdot\mid\tau_{h})} \left[\up Q_h^{k,\pi}(\tau_h,\a) \right],\\
\up Q_h^{k,\pi}(\tau_h,\a) =  \min\left\{ \E_{s' \sim\hat\P^k_h(\cdot \mid s_{h},\a)} \left[\up V_{h+1}^{k,\pi}([\tau_h,\a,s']) \right] +r_h(s_h,\a)+ \beta(N_h^k(s_h,\a)), H-h+1\right\},
\end{cases}
\end{equation}
and we define $\up V_{H+1}^{k,\pi} \equiv 0$. We comment that by definition $\up{V}_1^{k,\pi}(s_1)=\text{UCB-VI}(N^{k,{\rm lazy}},\beta,\pi)$ for all $k,\pi$.

For the purpose of proof, we further introduce the following auxiliary function for controlling the optimism of $\up{V}^{\pi,k}$ ($\up{Q}^{\pi,k}$) against the true value function $V^\pi$ ($Q^\pi$): for any $h\in[H]$
\begin{equation}
\begin{cases}
\tilde V_h^{k,\pi}(\tau_h) = \E_{\a\sim\pi(\cdot\mid\tau_h)} \left[\tilde Q_h^{k,\pi}(\tau_h,\a) \right],\\
\tilde Q_h^{k,\pi}(\tau_h,\a) =  \min\left\{ \E_{s' \sim\P_h(\cdot \mid s_{h},\a)} \left[\tilde V_{h+1}^{k,\pi}([\tau_h,\a,s']) \right] +2\beta(N^{k,{\rm lazy}}_h(s_h,\a)), H-h+1\right\},
\end{cases}
\end{equation}
and we define $\tilde V_{H+1}^{k,\pi} \equiv 0$. 
Compared to $\up{V}^{k,\pi}$, $\tilde V_h^{k,\pi}$ is defined using the groundtruth transition function $\P$, it does not contain the reward function and the bonus function is doubled.

Finally, recall we choose 
$$
\beta(t) = H\sqrt{\frac{S\iota}{\max\{t,1\}}},
$$
where  $\iota =\log(KHSA/\delta)$ with $c$ being some large absolute constant.

\begin{lemma}[Optimism]\label{lem:optimism-2}
    With probability at least $1-\delta$, for all $(k,h)\in[K]\times[H+1]$ and all general policy $\pi$, 
    $$
    0 \le \up V^{k,\pi}_{h}(\tau_h) - V^{\pi}_{h}(\tau_h) \le \tilde V^{k,\pi}_{h}(\tau_h) \quad \text{for all }\tau_h.
    $$
    \end{lemma}
\begin{proof}
The proof of Lemma \ref{lem:optimism-2} follows exactly the same as that of Lemma \ref{lem:optimisL-1} except that we replace $N^k$ with $N^{k,{\rm lazy}}$.
\end{proof}

\begin{proof}[Proof of Theorem \ref{thm:up-2}]
    In the remainder of this section, we show how to control $\Reg(K)$. The upper bound for $\Reg(k)$ ($k\in[K]$) can be derived by repeating precisely the same arguments.

    Denote by $\Phi^k$ ($\Psi^k$) the player (opponent) policy set at the \emph{beginning} of episode $k$. 
    Recall in Algorithm \ref{alg:md-restart}, each time we encounter a new opponent policy or one of the counters is doubled, we update the lazy counters to be the latest counters, recompute the player policy set, and restart EXP3 from the uniform distribution. 
    We denote the indices of episodes where such restarting happens by $T_1,\ldots,T_L$. Observe that 
    $L \le \cO(SAH\log(K)+ |\Psi^K| )$. 
    
    To begin with, we decompose the cumulative regret of $K$ episodes into the regret within $L+1$ segments divided by $T_1,\ldots,T_L$: 
    \begin{equation}\label{eq:up-0}
        \begin{aligned}
            & \max_\mu \sum_{k=1}^K \left( V^{\mu\times\nu^k}_1(s_1) 
         - V^{\mu^k\times\nu^k}_1(s_1) \right)\\
         \le  & \sum_{i=1}^{L-1} \max_\mu \sum_{k=T_i+1}^{T_{i+1}-1}  \left( V^{\mu\times\nu^k}_1(s_1) 
         - V^{\mu^k\times\nu^k}_1(s_1) \right) + \max_\mu \sum_{k=T_L+1}^{K}  \left( V^{\mu\times\nu^k}_1(s_1) 
         - V^{\mu^k\times\nu^k}_1(s_1) \right) \\
         &\quad  +\max_\mu \sum_{i=1}^L   \left( V^{\mu\times\nu^{T_i}}_1(s_1) 
         - V^{\mu^{T_i}\times\nu^{T_i}}_1(s_1) \right)\\
          \le  &\sum_{i=1}^{L-1} \max_\mu \sum_{k=T_i+1}^{T_{i+1}-1}  \left( V^{\mu\times\nu^k}_1(s_1) 
         - V^{\mu^k\times\nu^k}_1(s_1) \right) + \max_\mu \sum_{k=T_L+1}^{K}  \left( V^{\mu\times\nu^k}_1(s_1) 
         - V^{\mu^k\times\nu^k}_1(s_1) \right)+HL.
        \end{aligned}
    \end{equation}
        Below we show how to control $\sum_{k=T_i+1}^{T_{i+1}-1}  \left( V^{\mu\times\nu^k}_1(s_1) 
        - V^{\mu^k\times\nu^k}_1(s_1) \right)$ for any $i\in[L-1]$. The second term can be bounded in the same way. By Lemma \ref{lem:optimism-2}, with probability at least $1-2\delta$,
        \begin{equation}\label{eq:up-1}
            \begin{aligned}
        & \max_\mu \sum_{k=T_i+1}^{T_{i+1}-1}  \left( V^{\mu\times\nu^k}_1(s_1) 
        - V^{\mu^k\times\nu^k}_1(s_1) \right)\\
         \le  &  \max_\mu \sum_{k=T_i+1}^{T_{i+1}-1}  \left( \up V^{k,\mu\times \nu^k}_1(s_1) 
        - \up V^{k,\mu^k\times \nu^k}_1(s_1) \right) +   \sum_{k=T_i+1}^{T_{i+1}-1}  \left( \up V^{k,\mu^k\times \nu^k}_1(s_1) 
        - V^{\mu^k\times\nu^k}_1(s_1) \right)\\
        \le & \max_\mu \sum_{k=T_i+1}^{T_{i+1}-1}  \left(\up  V^{k,\mu\times \nu^k}_1(s_1) 
        - \E_{\mu'\sim \bp^k}\left[\up  V^{k,\mu'\times \nu^k}_1(s_1)\right] \right) + \cO\left(H\sqrt{(T_{i+1}-T_i-1)\iota}\right)\\
        & \qquad \qquad \qquad +   \sum_{k=T_i+1}^{T_{i+1}-1}  \left( \up V^{k,\mu^k\times \nu^k}_1(s_1) 
        - V^{\mu^k\times\nu^k}_1(s_1) \right),
        \end{aligned}
    \end{equation}
        where in the second inequality we use the Azuma-Hoeffding inequality and take a union bound for all the possible values of $T_i$ and $T_{i+1}$. Specifically, we use the fact that with probability at least $1-\delta$, for all $p,q\in[K]$, 
        $$
        \sum_{k=p+1}^{q-1} \E_{\mu'\sim \bp^k}\left[\up  V^{k,\mu'\times \nu^k}_1(s_1)\right] - \up  V^{k,\mu^k\times \nu^k}_1(s_1)\le \cO\left(H\sqrt{(q-p-1)\iota}\right).
        $$
    
    The key to controlling the RHS of equation \eqref{eq:up-1} is to show that the first term is approximately upper bounded by the regret of EXP3. 
        Recall that for $k$ lying between $T_i$ and $T_{i+1}$, the opponent does not play any new policy and the lazy counter is never updated. As a result, for all $k$ satisfying $T_i< k<T_{i+1}$, we have
        \begin{itemize}
            \item $\up{V}^{k,\pi}= \up{V}^{T_i+1,\pi}$ for all $\pi$.
            \item $\Phi^k=\Phi^{T_i+1}$, $\Psi^k=\Psi^{T_i+1}$, and $\nu^k\in\Psi^{T_i+1}$.
        \end{itemize}
        Moreover, by the definition of the UCB-BestResponse subroutine, we have that for any policy $\tilde \nu$ that is a mixture of the policies in $\Psi^{T_i+1}$, there exists $\tilde\mu\in \Phi^{T_i+1}$ so that 
        $$
        \max_\mu \up{V}^{T_i+1,\mu\times\tilde \nu}_1(s_1) - 
        \up{V}^{T_i+1,\tilde \mu\times\tilde \nu}_1(s_1) \le \epsilon H.
        $$
        By utilizing the three relations above, we have
        \begin{equation}\label{eq:up-2}
            \begin{aligned}
            &\max_\mu \sum_{k=T_i+1}^{T_{i+1}-1}  \left(\up  V^{k,\mu\times \nu^k}_1(s_1) 
        - \E_{\mu'\sim \bp^k}\left[\up  V^{k,\mu'\times \nu^k}_1(s_1)\right] \right)\\
    =  & \max_\mu \sum_{k=T_i+1}^{T_{i+1}-1}  \left(\up  V^{T_i+1,\mu\times \nu^k}_1(s_1) 
    - \E_{\mu'\sim \bp^k}\left[\up  V^{T_i+1,\mu'\times \nu^k}_1(s_1)\right] \right) \\
    \le & \max_{\tilde \mu \in \Phi^{T_i+1}}
    \sum_{k=T_i+1}^{T_{i+1}-1}  \left(\up  V^{T_i+1,\tilde\mu\times \nu^k}_1(s_1) 
    - \E_{\mu'\sim \bp^k}\left[\up  V^{T_i+1,\mu'\times \nu^k}_1(s_1)\right] \right)+ \left(T_{i+1}-T_i-1\right)\epsilon H \\
    \le & \cO\left( \sqrt{\log\left|  \Phi^{T_i+1}\right|\left(T_{i+1}-T_i-1\right)H^2}+ \left(T_{i+1}-T_i-1\right)\epsilon H \right),
    \end{aligned}
        \end{equation}
    where the first inequality follows from $\frac{1}{T_{i+1}-T_i-1}\sum_{k=T_i+1}^{T_{i+1}-1} \nu^k$ being a mixture of policies in $\Psi^{T_i+1}$, and the second inequality follows from Algorithm \ref{alg:md-restart} running anytime EXP on $\Phi^{T_i+1}$ and using  $-\up  V^{T_i+1}$ as  gradients for iterations between $T_i$ and $T_{i+1}$.
            
    Finally, combining equations \eqref{eq:up-0}, \eqref{eq:up-1}, and \eqref{eq:up-2} together, we obtain
        \begin{align*}
            & \max_\mu \sum_{k=1}^K \left( V^{\mu\times\nu^k}_1(s_1) 
         - V^{\mu^k\times\nu^k}_1(s_1) \right)\\
         \le &\min\left\{ \cO\left(HL + H\sqrt{KL\iota} + KH\epsilon +\sqrt{KLH^2\log\left|  \Phi^{K}\right|} \right), KH\right\}
         \\&\qquad\qquad + \sum_{k=1}^{k}  \left( \up V^{k,\mu^k\times \nu^k}_1(s_1) 
         - V^{\mu^k\times\nu^k}_1(s_1) \right).
        \end{align*}
    For the second term, note that $N^{k,{\rm lazy}}_h(s,\a)= \Theta(N^{k}_h(s,\a))$ for all $(s,\a,h,k)$, so following the identical arguments in the proof of Theorem \ref{thm:up-1} gives: with probability at least $1-\delta$
    $$
    \sum_{k=1}^{K}  \left( \up V^{k,\mu^k\times \nu^k}_1(s_1) 
         - V^{\mu^k\times\nu^k}_1(s_1) \right)\le \cO\left( \sqrt{KS^2AH^4\iota^2}\right).
    $$
    For the first term, plug in $\epsilon=1/K$ and notice that $\min\{HL,HK\} \le  H\sqrt{KL\iota}$ as well as  $\log\left|  \Phi^{K}\right|\le \cO( |\Psi^K| \log(K))$,
    \begin{align*}
    &\min\left\{ \cO\left(HL + H\sqrt{KL\iota} + KH\epsilon + \sqrt{KLH^2\log\left|  \Phi^{K}\right|} \right), KH\right\}\\
    \le &  \cO\left( H\sqrt{KL\iota}  + \sqrt{KL |\Psi^K| H^2\iota} \right)\\
    \le &  \cO\left( H\sqrt{K |\Psi^K| SAH\iota^2} + H\sqrt{K |\Psi^K| ^2 \iota}  \right)  ,
    \end{align*}
    where the second inequality uses $L \le \cO(SAH\log(K)+ |\Psi^K| )$. Combining all relations, we conlude the final upper bound is 
    $$
    \cO\left( \sqrt{KS^2AH^4\iota^2} + \sqrt{K |\Psi^K| SAH^3\iota^2} + \sqrt{K |\Psi^K| ^2 H^2\iota} \right).
    $$    \end{proof}

    \subsection{Proof of Theorem~\ref{thm:general-lower}}\label{pf:general-lower}

\begin{proof}
	Consider the following Markov game with one state $s$ and horizon $H$. The action set for both the max-player and the min-player (adversary) is $\{0,1\}$. For $h=1,\cdots,H-1$, $r_h(s,\cdot)=0$. $r_H(s,(a,b))=I[a=b]$.
	
	Suppose that at episode $t$, the adversary chooses a policy $\nu_t$ which plays $b^t_h$ at time step $h$, where each $b^t_h$ is sampled independently from ${\rm Unif}(\{0,1\})$.\footnote{The policy itself is deterministic.} It can be easily seen that for each episode, the expected value for the max-player is always $\frac{1}{2}$. However, the best general policy in hindsight will be able to predict $b_H$ from $b_{1:h-1}$ to a large extent. Specifically, for each possible value of $b_{1:h-1}$, define $N(b_{1:h-1}):=\sum_{t=1}^T I[b^t_{1:h-1}=b_{1:h-1}]$. If $T<2^{H-2}$,
	$$\Pr\left[N(b_{1:h-1})>1\right]\le \frac{\E[N(b_{1:h-1})]}{2}\le\frac{1}{4}.$$
	If $N(b_{1:h-1})=1$, denote the only episode in which it appears by $t$. The best general policy in hindsight could set $\mu(s,b_{1:h-1})=b^t_H$ and achieve value $1$ on episode $t$. In other words, there exists general policy $\mu$ such that 
	\begin{align*}
		\E\left[\sum_{t=1}^{T}V^{\mu, \nu_t}_1(s)\right] \ge \E\left[\sum_{t=1}^{T}I[N(b^k_{1:h-1})=1]\right] \ge \frac{3}{4}T.
	\end{align*}
	Therefore regret is at least $\frac{1}{4}T$, unless $T\ge 2^{H-2}$.
\end{proof}

\section{Proofs for Section \ref{sec:comp}}
\label{sec:lower-bound-proof}

\subsection{Proof of Theorem~\ref{thm:computation-lower}}\label{pf:comp}

	The proof of this theorem is essentially a reduction to Proposition $6$ of~\citet{steimle2021multi}. We present a full proof here for the sake of clarity and completeness.
	
	We would prove the theorem via reduction from 3-SAT. Consider a 3-SAT instance with $m$ clauses and $n$ variables: $\land_{i=1}^m (y_{i1}\lor y_{i2}\lor y_{i3})$, where $y_{ij}\in \{x_1,\cdots,x_n,\bar x_1,\cdots,\bar x_n\}$. We would then construct a Markov game with $H=n$, $|\cA_{\max}|=2$ and $|\cA_{\min}|=m$ as follows. The set of states are $\{s_1,\cdots,s_n,{\rm T},{\rm F}\}$. The action set is $\{0,1\}$  for the max-player  and $[m]$ for the min-player. The transitions are deterministic, independent of $h$, and are specified as follows:
	\begin{align*}
		\P({\rm T}|s_i, (a, j)) &= \begin{cases}
			1 & (\text{setting }x_i=a \text{ sets clause }j\text{ to True})\\
			0 & (\text{otherwise})
		\end{cases}\\
		\P(s_{i+1}|s_i, (a, j)) &= 1-\P({\rm T}|s_i, (a, j)),\tag{i<n}\\
		\P({\rm F}|s_n, (a, j)) &= 1-\P({\rm T}|s_n, (a, j)).
	\end{align*}
	For $h<n$, $r_h(\cdot,\cdot)=0$; for $h=n$, $r_n({\rm T},\cdot)=1$, $r_n({\rm F},\cdot)=0$. Every Markov policy $\mu$ of the max-player induces an assignment of the variables, \emph{i.e.} $x_i=\mu(s_i)$. Moreover, denote the min-player's policy of playing action $j$ at all states by $\nu_j$. Notice that
	\begin{equation*}
		V_1^{\mu, \nu_j}(s_1)=\begin{cases}
			1& (\text{assignment induced by $\mu$ satisfies clause $j$})\\
			0& (\text{assignment induced by $\mu$ violates clause $j$})
		\end{cases}.
	\end{equation*}
	If an algorithm achieves  $\E[{\rm Regret}(T)]={\rm poly}(S,A,H)\cdot T^{1-c}$ regret, then there exists $T=\poly(n,m)$ such that $T\ge \max\left\{4m\cdot \E[{\rm Regret}(T)], 20m\sqrt{T}\right\}$. Now consider the following algorithm for 3-SAT:
	
	\begin{enumerate}
		\item Construct the aforementioned Markov game
		\item Run algorithm $\mathcal{A}$, with the opponent playing $\nu_j$ with $j$ sampled from ${\rm Unif}([m])$ at the start of each episode independently
		\item Calculate $R$, the total reward accumulated by the algorithm. Decide True (satisfiable) if $R>(1-1/2m)T$, and False otherwise.
	\end{enumerate}
	
	We now claim that if the input instance is satisfiable, this algorithm returns True with probability at least $0.99$. This is because satisfiability implies $\exists \mu^*:$, $\sum_{t=1}^T V^{\mu^*, \nu_t}(s_1)=T.$ By the definition of regret,
	$$\E\left[\sum_{t=1}^T \left(V^{\mu^*, \nu_t}-V^{\mu_t,\nu_t}\right)(s_1)\right]\le \frac{T}{4m}.$$
	By Hoeffding's inequality, with probability at least $0.99$,
	$$R\ge \E\left[\sum_{t=1}^T V^{\mu_t, \nu_t}\right] - 5\sqrt{T}>T-\frac{T}{4m}-\frac{T}{4m}=\left(1-\frac{1}{2m}\right)T.$$
	Meanwhile, if the input instance is insatisfiable, then with probability $0.99$, the algorithm returns False. This is because with probability $0.9$,
	\begin{align*}
		\E_{t\sim [T], j\sim [m]}\left[\sum_{t=1}^T V^{\pi_t,\nu_j}\right] \ge R-5\sqrt{T} > R-\frac{T}{4m}.
	\end{align*}
	Conditioned on this event, if the algorithm returns True, then
	$$\E_{t\sim [T], j\sim [m]}\left[\sum_{t=1}^T V^{\pi_t,\nu_j}\right]\ge T-\frac{T}{2m}-\frac{T}{4m}> \left(1-\frac{1}{m}\right)T.$$
	This implies that $\exists t$:
	$$\E_{j\sim [m]}\left[\sum_{t=1}^T V^{\pi_t,\nu_j}\right]>\left(1-\frac{1}{m}\right)T,$$
	which contradicts with the fact that the input is not satisfiable. Therefore the probability that the algorithm returns True when the input is not satisfiable is at most $0.01$.
	
	The reduction above suggests that, if algorithm $\mathcal{A}$ runs in ${\rm poly}(S,A,H,T)$ time, we can obtain an algorithm that decides $3$-SAT with high probability and runs in ${\rm poly}(m,n)$ time. In other words, this implies 3-SAT$\in$BPP, which further implies NP$\subseteq$BPP since $3$-SAT is NP-complete.

\end{document}